\documentclass[11pt]{article}
\usepackage[top=1in, bottom=1in, left=1in, right=1in]{geometry}

\usepackage[colorinlistoftodos]{todonotes}
\usepackage{url}
\usepackage{hyperref}
\usepackage{natbib}
\usepackage{booktabs}
\usepackage{mathtools}
\usepackage{subcaption}
\usepackage{comment}
\usepackage{graphicx} % Required for inserting images
\usepackage{amsmath,amssymb}
\usepackage{float}
\usepackage{authblk}
\usepackage{amsthm}
\usepackage{amsthm}
\usepackage{tikz-cd}
\usepackage{verbatim}
\usepackage{xcolor}
\usepackage[colorinlistoftodos]{todonotes}
\newtheorem{theorem}{Theorem}
\newtheorem{lemma}[theorem]{Lemma}
\newtheorem{proposition}[theorem]{Proposition}
\newtheorem{corollary}[theorem]{Corollary}

\theoremstyle{definition}
\newtheorem{definition}[theorem]{Definition}
\usepackage{graphicx} % Required for inserting images

\title{Representation Redundancy and Structural Complexity in Finite-Field Inversion}
\author{
Zheng Zhang \thanks{Corresponding author: Department of Mathematics, Towson University, 7800 York Rd, Towson, MD 21204, USA. Email: \texttt{zhengzhang@towson.edu}.},
Na Zhang \thanks{Department of Mathematics, Towson University, 7800 York Rd, Towson, MD 21204, USA. Email: \texttt{nzhang@towson.edu}.}
}
\date{}
\date{}

\begin{document}

\maketitle


\abstract{
The representation chosen for a mathematical operation can affect both its
algebraic form and its empirical learning difficulty. We study this phenomenon
for inversion over \(\mathbb F_{2^n}\), with field elements expressed in
varying ordered \(\mathbb F_2\)-bases. We prove that two ordered bases induce
the same coordinate inversion map if and only if they belong to the same
Galois orbit. Since every orbit has size \(n\), the correspondence between
ordered bases and distinct inversion maps is exactly \(n\)-to-one. We then
analyze three Boolean formulations of inversion. The reference formulation
has algebraic degree \(n-1\) and joint ANF leap \(1\), the mixed representation
formulation has degree \(2(n-1)\) and joint ANF leap \(2\), and the complete
raw formulation has degree at most \(3(n-1)\) and joint ANF leap at least
\(n\). Exhaustive computations agree with the theoretical results and bounds in the
cases considered. Controlled experiments with multilayer perceptrons show the same
ordering in learning difficulty, while Galois orbit redundancy provides only
a limited generalization benefit under the tested conditions. These results show that exact redundancy among representations can coexist
with changes in Boolean structure and learning behavior when
the representation is exposed as part of the input.}

% Keywords

\maketitle


\section{Introduction}\label{sec:introduction}
Many mathematical and learning problems admit more than one coordinate representation of the same underlying object. An element of a finite field, a point on a manifold, an image under a group action, or a physical state expressed in different bases may all be encoded by different input vectors while representing the same object. Although the underlying operation remains unchanged, its coordinate form can depend strongly on the chosen representation. This raises two basic questions: how much of the variation among representations is redundant, and how does the choice of representation affect the algebraic structure and learnability of the resulting map?

\subsection{Symmetry and Sample Complexity}
\label{subsec:complementary-gap}
A central motivation for equivariant learning is that exploiting symmetry and
redundancy among representations can improve sample efficiency and
generalization. This may be achieved by
encoding a known group action in the architecture, sharing parameters
across inputs connected by the symmetry, or constructing invariant
representations~\cite{CohenWelling2016,Esteves2020,Bronstein2021}.

Several works place these benefits on rigorous statistical foundations.
Generalization bounds for invariant classifiers can depend on the complexity
of the quotient space induced by the transformations, which may be much
smaller than that of the full input space~\cite{SokolicEtAl2017}. Strict
generalization improvements have also been established for equivariant linear
models under suitable distributional assumptions~\cite{ElesedyZaidi2021}.
Invariant kernels can improve sample complexity by a factor related to the
size of the transformation group, with the finite-sample gain also depending
on the structure of the group~~\cite{BiettiVenturiBruna2021}. Related
separation results quantify the benefits of architectural biases such as
locality and weight sharing by comparing convolutional, locally connected,
and fully connected networks~\cite{LahotiEtAl2024}.

These works assume a known transformation structure or impose a corresponding
architectural constraint, and then quantify its statistical benefits within a
specified learning model. Our focus is complementary. We study a finite and exactly enumerable
setting in which two properties can be analyzed separately. The first is the
exact redundancy induced by representation choice. The second is the
algebraic structure of the joint task obtained when the representation itself
is exposed as an input variable.

\subsection{Redundancy and Algebraic Complexity}
\label{subsec:two-questions}

We study two aspects of representation choice as separate, precisely
defined questions.

\medskip
\noindent
\textbf{Question 1 (Exact redundancy).}
Given a family of coordinate representations related by a known symmetry
group, when do two representations induce exactly the same function, rather
than agreeing only approximately or with high probability? How large are the
resulting equivalence classes?

\medskip
\noindent
\textbf{Question 2 (Algebraic complexity).}
When the representation is included as part of the input in addition to the
operand, how does it change the algebraic structure of the resulting Boolean
map?

This question does not concern the computational complexity of finite-field
inversion, since the underlying operation remains unchanged. It concerns the
ANF structure of the joint map from the representation and the operand to the
output.

The two questions are defined at different mathematical levels.
Exact redundancy concerns equivalence among the maps associated with different
representations, whereas algebraic complexity concerns the interactions
between the representation variables and the operand variables in the joint
map. Neither analysis can replace the other.

\subsection{Finite-Field Inversion as an Exactly Solvable Model}
\label{subsec:finite-field-model}

We study these questions through multiplicative inversion in the finite field
\(F=\mathbb F_{2^n}\), expressed in coordinates relative to a varying ordered
basis over \(\mathbb F_2\). Finite-field inversion is a classical example of a
vectorial Boolean function with important cryptographic properties
\cite{Nyberg1994,Carlet2021}. It also forms the nonlinear core of the AES
S-box, where inversion in \(\mathbb F_{2^8}\) is followed by an affine
transformation~\cite{DaemenRijmen2020}.

Our earlier work studied finite-field multiplication under changes of basis,
with an emphasis on symmetry, invariance, and learning across Galois equivalent
representations~\cite{zhang2026symmetryinvariancelearninggalois}. The present work turns to inversion and
studies both the exact redundancy among its coordinate maps and the change in
ANF structure when the basis is included as part of the input. We choose
inversion as a model for three reasons.

First, it provides a finite and explicit space of representations.
The admissible representations are precisely the ordered
$\mathbb F_2$-bases of $F$, and for small $n$, the entire representation
space can be enumerated, so sampling is unnecessary.

Second, its representation redundancy can be characterized exactly.
The Galois group of \(F/\mathbb F_2\) is cyclic of order \(n\) and acts
naturally on ordered bases through the Frobenius automorphism. This suggests that bases in the same Galois orbit define equivalent tasks.
The central question is whether belonging to the same orbit is also necessary
for two bases to induce the same coordinate inversion map.

Third, the structural complexity introduced by the choice of representation
can be analyzed algebraically. In binary coordinates, inversion becomes a vector-valued Boolean map whose
algebraic normal form can be studied directly. This allows us to determine how the interactions
among the input variables change when the basis is fixed, used to transform
the input, or included as part of the raw input. For small fields, the
resulting structural quantities can also be verified by exhaustive computation
over the full input space.

\subsection{Two Theorem Chains}
\label{subsec:theorem-chains}

We answer the two questions through two complementary chains of
results.

\medskip
\noindent
\textbf{Exact redundancy.}
We give an exact characterization of task equivalence. Two ordered bases
induce the same coordinate inversion map if and only if they belong to the
same Galois orbit. Every such orbit contains exactly $n$ bases, so each
inversion task has exactly $n$ different basis representations. Equivalently,
the number of distinct inversion maps is the number of ordered bases divided
by $n$. Thus, basis variation produces an exact $n$-to-one redundancy.

\medskip
\noindent
\textbf{Algebraic complexity.}
We also characterize how including the basis as part of the input changes the
algebraic complexity of inversion. The reference inversion map has degree
$n-1$, while the mixed representation map $J_n(P\mathbf u)$ has degree
$2(n-1)$. Thus, introducing the basis matrix at the input transformation
stage doubles the algebraic degree. The complete raw map
$f_{\mathrm{raw}}$, which also includes the inverse basis transformation, has
degree at most $3(n-1)$.

To describe aspects of the ANF structure that are not captured by degree, we
introduce the \emph{joint ANF leap}. This definition adapts the leap and
staircase ideas of Abbe et al.~\cite{AbbeEtAl2022,AbbeEtAl2023} from the
Fourier support of scalar-valued functions to the joint ANF support of
vector-valued Boolean maps. We prove that the joint ANF leap is $1$
for the reference map, $2$ for the mixed representation map, and at least
$n$ for the complete raw map. These results distinguish two complementary aspects of ANF
structure: interaction order and stepwise variable introduction. We later compare this measure with the observed learning behavior. A formal
connection to ANF based learnability is left for future work.

\paragraph{Redundancy and accessibility from raw inputs.}
The contrast between the two chains is a central point of the paper. The
redundancy result shows that variation among representations creates an exact
redundancy by a factor of $n$. The complexity results show that exposing the representation changes the
ANF structure of the task, both in its interaction order and in how its
variables must be introduced across the joint ANF support. The existence of redundancy does not imply that a
model trained on raw inputs can readily use it. Our experiments examine
whether this distinction is reflected in the observed learning behavior.

\subsection{Experiments}
We use four experiments to examine the theoretical results. The first verifies
the ANF structure for $n=3$ and $n=4$ by exhaustive computation. The remaining three
study learning across the three formulations, generalization across Galois
orbits, and scaling with training data and model width.

\paragraph{Experiment 0: Exact ANF computation.}
We compute the exact ANF structure of the reference, mixed representation,
and complete raw formulations for $n=3$ and $n=4$, and compare the results
with the theoretical statements and bounds.

\paragraph{Experiment 1: Learning across the three formulations.}
We compare the three formulations under matched experimental conditions
to examine whether the ordering suggested by their ANF structure is also
reflected in the observed learning performance.

\paragraph{Experiment 2: Generalization and representation redundancy.}
We compare random-basis and orbit-disjoint splits into training and
test sets to examine the practical effect of the exact representation
redundancy characterized by the first theorem chain.

\paragraph{Experiment 3: Scaling with data and model width.}
We vary the amount of training data and model width to examine whether the
empirical differences among the three formulations persist as data and model
capacity increase.

%Together, these experiments separately examine the empirical manifestations
%of representation redundancy and representation-exposed structural
%complexity.

\subsection{Contributions}
\label{subsec:contributions}

Our main contributions are summarized as follows.

\begin{itemize}

\item \textbf{Exact representation redundancy.}
We prove that two ordered bases induce the same coordinate inversion map if
and only if they belong to the same Galois orbit. Since each orbit has size
$n$, every inversion task has exactly $n$ different basis representations.

\item \textbf{Algebraic complexity under representation change.}
We characterize how the ANF structure changes when the basis is included at
different stages of the inversion map. The algebraic degrees of the reference,
mixed representation, and complete raw formulations satisfy
\[
n-1
\;\longrightarrow\;
2(n-1)
\;\longrightarrow\;
\leq 3(n-1).
\]
We also introduce the \emph{joint ANF leap}, which adapts the leap ordering
principle from the Fourier support of scalar-valued functions to the joint ANF
support of vector-valued Boolean maps. For
the three formulations, the corresponding values satisfy
\[
1
\;\longrightarrow\;
2
\;\longrightarrow\;
\geq n.
\]
These results describe two complementary aspects of ANF structure:
interaction order and stepwise variable introduction.

\item \textbf{Computational and empirical verification.}
We verify the ANF structure for $n=3$ and $n=4$ by exhaustive computation. We then use controlled learning experiments to compare the three
formulations, study generalization across Galois orbits, and examine scaling
with training data and model width. The observed learning behavior is
consistent with the ordering given by the ANF analysis.

\end{itemize}

\section{Coordinates over Finite Fields and Galois Orbits}
\label{sec:math-frame}
This section develops the coordinate framework needed to characterize
representation redundancy in finite-field inversion. We first define the
reference and basis-dependent inversion maps and relate them through a change
of basis. We then
introduce the Frobenius action on ordered bases and prove that its Galois
orbits are exactly the equivalence classes of bases that induce the same
inversion map. This gives an exact count of the redundancy created by basis
variation.

\subsection{Finite-field coordinates and change of basis}
\label{subsec:coordinates}
Let
\[
F=\mathbb F_{2^n}.
\]
A standard construction of $F$ is
\[
F=\mathbb F_2[t]/(p(t)),
\]
where $p(t)\in\mathbb F_2[t]$ is irreducible of degree $n$~\cite{LidlNiederreiter1997}. Let
$\alpha$ denote the residue class of $t$ modulo $p(t)$. Every element
$x\in F$ can be written uniquely as
\[
x=x_0+x_1\alpha+\cdots+x_{n-1}\alpha^{n-1},
\qquad x_i\in\mathbb F_2.
\]
Hence
\[
E=(1,\alpha,\ldots,\alpha^{n-1})
\]
is an ordered basis of $F$ over $\mathbb F_2$. Throughout the paper,
we fix $E$ as the reference basis.

Let
\[
V=\mathbb F_2^n.
\]
For $x\in F$, we denote its coordinate vector relative to $E$ by
\[
[x]_E=(x_0,\ldots,x_{n-1})^{\mathsf T}\in V.
\]
We extend inversion to all of $F$ by adopting the convention
\[
0^{-1}:=0.
\]

\begin{definition}[Reference inversion map]
\label{def:reference-inversion}
The reference inversion map is the function
\[
J_n:V\longrightarrow V
\]
defined by
\[
J_n([x]_E)=[x^{-1}]_E,
\qquad x\in F.
\]
\end{definition}
\(J_n\) is the coordinate realization of field inversion in the
fixed reference basis \(E\). It will serve as the baseline with a fixed representation throughout the
paper.

Now let
\[
B=(b_0,\ldots,b_{n-1})
\]
be an arbitrary ordered basis of $F$ over $\mathbb F_2$. Every
$x\in F$ has a unique representation
\[
x=u_0b_0+\cdots+u_{n-1}b_{n-1},
\qquad u_i\in\mathbb F_2.
\]
We denote its coordinate vector relative to $B$ by
\[
[x]_B=(u_0,\ldots,u_{n-1})^{\mathsf T}\in V.
\]

\begin{definition}[Basis-dependent inversion map]
\label{def:basis-inversion}
For an ordered basis $B$ of $F$ over $\mathbb F_2$, the
basis-dependent inversion map is the function
\[
\operatorname{Inv}_B:V\longrightarrow V
\]
defined by
\[
\operatorname{Inv}_B([x]_B)=[x^{-1}]_B,
\qquad x\in F.
\]
\end{definition}

In particular,
\[
J_n=\operatorname{Inv}_E.
\]

Although \(J_n\) and \(\operatorname{Inv}_B\) are both maps from
\(V\) to \(V\), their coordinates have different meanings. The input
and output of \(J_n\) are interpreted relative to \(E\), whereas those
of \(\operatorname{Inv}_B\) are interpreted relative to \(B\).
Changing the basis therefore changes the coordinate map even though
the underlying field operation \(x\mapsto x^{-1}\) remains the same. To compare these coordinate realizations within a common reference
system, we introduce the matrix that converts \(B\)-coordinates into
\(E\)-coordinates.

\begin{definition}[Change-of-basis matrix]
\label{def:change-of-basis}
For an ordered basis
\[
B=(b_0,\ldots,b_{n-1}),
\]
define the change-of-basis matrix from $B$-coordinates to
$E$-coordinates by
\[
P_B
=
\begin{bmatrix}
[b_0]_E &[b_1]_E&\cdots &[b_{n-1}]_E
\end{bmatrix}
\in\operatorname{GL}_n(\mathbb F_2).
\]
\end{definition}

The $j$th column of $P_B$ is the coordinate vector of $b_j$ relative
to $E$. Hence, for every $x\in F$,
\[
[x]_E=P_B[x]_B.
\]
It follows that
\[
\begin{aligned}
\operatorname{Inv}_B([x]_B)
&=[x^{-1}]_B\\
&=P_B^{-1}[x^{-1}]_E\\
&=P_B^{-1}J_n([x]_E)\\
&=P_B^{-1}J_n(P_B[x]_B).
\end{aligned}
\]
Therefore, for every $\mathbf u\in V$,
\begin{equation}
\operatorname{Inv}_B(\mathbf u)
=
P_B^{-1}J_n(P_B\mathbf u).
\end{equation}
We obtained the following lemma.
\begin{lemma}
\label{InvBPB}
For every ordered basis $B$ of $F$ over $\mathbb F_2$,
\[
\operatorname{Inv}_B
=
P_B^{-1}\circ J_n\circ P_B.
\]
\end{lemma}

Lemma~\ref{InvBPB} gives a direct procedure for
computing inversion in any basis \(B\). First, \(P_B\) converts the
input from the basis \(B\) to the reference basis. Next, \(J_n\)
computes the inverse in the reference basis. Finally, \(P_B^{-1}\)
converts the result back to the basis \(B\). Thus, the effect of
choosing \(B\) is completely described by its change-of-basis matrix
\(P_B\).

\subsection{Frobenius action and Galois orbits}
Subsection~\ref{subsec:coordinates} shows how each ordered basis determines
an inversion map through its change-of-basis matrix. We now ask when
two different bases determine exactly the same map. The natural
candidate for this equivalence is provided by the Galois action. 
The Galois group of $F$ over $\mathbb F_2$ is cyclic of order $n$:
\[
\operatorname{Gal}(F/\mathbb F_2)
=
\langle \sigma\rangle
=
\{\sigma^r:0\le r\le n-1\},
\]
where
\[
\sigma(x)=x^2,
\qquad x\in F
\]
is the Frobenius automorphism.

The Galois group acts componentwise on the set of ordered
$\mathbb F_2$-bases of $F$. For
\[
B=(b_0,\ldots,b_{n-1}),
\]
define
\[
\sigma^r(B)
=
\bigl(\sigma^r(b_0),\ldots,\sigma^r(b_{n-1})\bigr).
\]
Since $\sigma^r$ is an $\mathbb F_2$-linear automorphism of $F$,
$\sigma^r(B)$ is again an ordered $\mathbb F_2$-basis.

\begin{definition}[Galois orbit of a basis]
\label{def:galois-orbit}
The Galois orbit of an ordered basis $B$ is
\[
\mathcal O(B)
=
\{\sigma^r(B):0\le r\le n-1\}.
\]
\end{definition}

The following observation is fundamental to our analysis: two bases
in the same Galois orbit induce the same coordinate inversion map.

\begin{theorem}[Exact task equivalence]
\label{thm:exact-equivalence}
Let \(B\) and \(B'\) be ordered \(\mathbb F_2\)-bases of
$F$. Then
\[
\operatorname{Inv}_B=\operatorname{Inv}_{B'}
\quad\Longleftrightarrow\quad
B'=\sigma^r(B)
\]
for some \(r\in\{0,\ldots,n-1\}\). Equivalently,
\[
\operatorname{Inv}_B=\operatorname{Inv}_{B'}
\quad\Longleftrightarrow\quad
B'\in\mathcal O(B).
\]
\end{theorem}

\begin{proof}
Suppose first that $B'=\sigma^r(B)$. If
\[
x=\sum_{i=0}^{n-1}u_i b_i,
\]
then, since $\sigma^r$ fixes $\mathbb F_2$,
\[
\sigma^r(x)=\sum_{i=0}^{n-1}u_i\sigma^r(b_i).
\]
Hence,
\[
[\sigma^r(x)]_{B'}=[x]_B.
\]
Since $\sigma^r$ is a field automorphism, it commutes with inversion. It
follows that
\[
\operatorname{Inv}_{B'}([x]_B)
=
[\sigma^r(x)^{-1}]_{B'}
=
[\sigma^r(x^{-1})]_{B'}
=
[x^{-1}]_B
=
\operatorname{Inv}_B([x]_B).
\]
As $x$ ranges over $F$, the vector $[x]_B$ ranges over $\mathbb F_2^n$.
Therefore,
\[
\operatorname{Inv}_{B'}=\operatorname{Inv}_B.
\]

Now suppose that
\[
\operatorname{Inv}_{B'}=\operatorname{Inv}_B.
\]
Let $S:F\to F$ be the unique $\mathbb F_2$ linear bijection satisfying
$S(b_i)=b'_i$ for every $i$. Then
\[
[S(x)]_{B'}=[x]_B
\qquad
\text{for all }x\in F.
\]
For every nonzero $x\in F$, the equality of the two inversion maps gives
\[
\begin{aligned}
[S(x^{-1})]_{B'}
&=[x^{-1}]_B\\
&=\operatorname{Inv}_B([x]_B)\\
&=\operatorname{Inv}_{B'}([S(x)]_{B'})\\
&=[S(x)^{-1}]_{B'}.
\end{aligned}
\]
Hence,
\[
S(x^{-1})=S(x)^{-1}.
\]

In particular,
\[
S(1)=S(1)^{-1}.
\]
Since \(S(1)\ne0\) and \(F\) has characteristic two, this gives
\(S(1)=1\). Hua's theorem~\cite{Hua1949,Artin1957} now implies that \(S\) is a field
automorphism of \(F\). Hence
\[
S=\sigma^r
\]
for some \(r\in\{0,\ldots,n-1\}\). Consequently,
\[
B'
=
\bigl(S(b_0),\ldots,S(b_{n-1})\bigr)
=
\sigma^r(B).
\]
\end{proof}

The preceding theorem shows that the equivalence classes of
basis-dependent inversion tasks are precisely the Galois orbits of ordered
bases. We next determine the size of these classes. The Galois action is
free. If \(\sigma^r(B)=B\), then \(\sigma^r\) fixes every element of the
basis \(B\), and therefore every element of \(F\). Thus, \(\sigma^r\) must
be the identity. Since \(\operatorname{Gal}(F/\mathbb F_2)\) has order
\(n\), every Galois orbit contains exactly \(n\) bases. The following
corollary gives the resulting exact count of representation redundancy.

\begin{corollary}[Exact basis redundancy]
\label{cor:exact-basis-redundancy}
The map
\[
B\longmapsto \operatorname{Inv}_B
\]
from ordered $\mathbb F_2$-bases of $F$ to coordinate inversion maps
is exactly $n$-to-one. Consequently, the number of distinct
basis-dependent inversion maps is
\[
\frac{|\operatorname{GL}_n(\mathbb F_2)|}{n}
=
\frac{1}{n}
\prod_{j=0}^{n-1}(2^n-2^j).
\]
\end{corollary}

\begin{proof}
By Theorem~\ref{thm:exact-equivalence}, two ordered bases induce
the same inversion map if and only if they belong to the same Galois
orbit. Every Galois orbit contains exactly $n$ ordered bases, while
the total number of ordered $\mathbb F_2$-bases of $F$ is
\[
|\operatorname{GL}_n(\mathbb F_2)|
=
\prod_{j=0}^{n-1}(2^n-2^j).
\]
The conclusion follows.
\end{proof}

The $\lvert\operatorname{GL}_n(\mathbb F_2)\rvert$ ordered bases induce
exactly $\lvert\operatorname{GL}_n(\mathbb F_2)\rvert/n$ distinct inversion
maps, giving an exact redundancy factor of $n$. After determining this
redundancy, we examine how including the basis transformation together with
the operand changes the algebraic structure of inversion.

\section{Algebraic Complexity of the Inversion Formulations}
This section compares three formulations of finite-field inversion: the
reference, mixed representation, and complete raw formulations. The basis
transformation is introduced in stages, first through input mixing and then
through output conversion. We study the resulting changes using algebraic
degree and joint ANF leap. Algebraic degree measures the maximum interaction order, while joint ANF leap measures how many new variables must be introduced at a single step
when the joint ANF support is optimally ordered.

\subsection{Algebraic Normal Form and Degree}
\label{subsec:anf-degree}

We begin by recalling the standard definitions of algebraic normal form and
algebraic degree for Boolean and vector-valued Boolean
functions~\cite{CusickStanica2017,Carlet2021}.

\begin{definition}[Algebraic normal form and algebraic degree]
\label{def:anf-degree}
Let
\[
f:\mathbb F_2^d\longrightarrow\mathbb F_2
\]
be a Boolean function, and write
\[
\mathbf x=(x_0,\ldots,x_{d-1})^{\mathsf T}\in\mathbb F_2^d.
\]
The \emph{algebraic normal form} (ANF) of $f$ is its unique multilinear
polynomial representation
\[
f(\mathbf x)
=
\bigoplus_{A\subseteq\{0,\ldots,d-1\}}
c_A\prod_{i\in A}x_i,
\qquad c_A\in\mathbb F_2.
\]
Here, $\bigoplus$ denotes addition in $\mathbb F_2$, equivalently XOR, and
the empty product corresponds to the constant term. The algebraic degree of
$f$ is
\[
\deg(f)
=
\max\bigl\{
|A|:c_A\neq0
\bigr\}.
\]
We use the convention $\deg(0)=-\infty$.

For a vector-valued Boolean map
\[
G=(G_0,\ldots,G_{m-1}):
\mathbb F_2^d\longrightarrow\mathbb F_2^m,
\]
we define
\[
\deg(G)
=
\max_{0\leq j\leq m-1}\deg(G_j).
\]
\end{definition}

We first determine the algebraic degree of inversion in the fixed reference
basis. This provides the baseline against which the mixed representation and
complete raw formulations will be compared. The proof uses a standard result
from the cryptographic theory of Boolean functions. The algebraic degree of a
finite-field power map is given by the binary Hamming weight of its
exponent~\cite{Carlet2021}. This result applies to $J_n$ because the reference
coordinates identify $F$ with $\mathbb F_2^n$ through an invertible
$\mathbb F_2$-linear map, which preserves algebraic degree.

\begin{theorem}[Degree of reference inversion]
\label{thm:degree-reference-inversion}
For $n\ge2$, the reference inversion map
\[
J_n:\mathbb F_2^n\longrightarrow\mathbb F_2^n
\]
has algebraic degree
\[
\deg(J_n)=n-1.
\]
\end{theorem}

\begin{proof}
It is standard that the algebraic degree of the finite-field power
map $x\mapsto x^s$ over $\mathbb F_{2^n}$ equals the binary Hamming
weight $w_2(s)$ of the exponent $s$~\cite{Carlet2021}. The field
inversion function underlying $J_n$ has the univariate representation
\[
x\longmapsto x^{2^n-2},
\]
where $0^{2^n-2}=0$. Since
\[
2^n-2=(11\cdots110)_2,
\]
we have
\[
w_2(2^n-2)=n-1.
\]
Therefore,
\[
\deg(J_n)=n-1.
\]
\end{proof}

We next include the basis transformation only at the input stage. Instead of
fixing a particular change-of-basis matrix $P_B$, we treat the entries of a
matrix $P$ jointly with $\mathbf u$ as Boolean input variables and consider
the map $J_n(P\mathbf u)$. This isolates the structural effect of converting
the input to the reference basis before converting the output back to the
original basis.

\begin{theorem}
[Degree of mixed representation inversion]
\label{thm:degree-mixed}
Let
\[
P=(p_{ij})_{0\le i,j\le n-1}
\]
be a variable matrix over \(\mathbb F_2\), and let
\[
\mathbf u=(u_0,\ldots,u_{n-1})^{\mathsf T}.
\]
Regarding the entries of \(P\) and the coordinates of \(\mathbf u\)
jointly as Boolean input variables, we have, for every \(n\ge2\),
\[
\deg\bigl(J_n(P\mathbf u)\bigr)=2(n-1).
\]
\end{theorem}

\begin{proof}
We prove the upper and lower bounds separately. For the upper bound, each coordinate of \(P\mathbf u\) is
\[
(P\mathbf u)_i
=
\sum_{j=0}^{n-1}p_{ij}u_j,
\]
which has degree \(2\) in the joint variables \((P,\mathbf u)\).
Each monomial
$$\prod_{i\in A}x_i$$ in a coordinate ANF of $J_n$ has
$|A|\leq n-1$. After substituting $\mathbf x=P\mathbf u$, this monomial
becomes
\[
\prod_{i\in A}(P\mathbf u)_i,
\]
which has degree at most
\[
2|A|\leq2(n-1).
\]
Hence,
\[
\deg\bigl(J_n(P\mathbf u)\bigr)\leq2(n-1).
\]

For the lower bound, specialize \(P\) by setting
\[
p_{ij}=0
\qquad\text{for all }i\ne j.
\]
This specialization is used only to obtain a lower bound. It is not a
without-loss-of-generality assumption on \(P\). Since restricting the
input variables cannot increase algebraic degree, it suffices to show
that the restricted function has degree \(2(n-1)\).

Since
\[
\deg(J_n)=n-1,
\]
some coordinate function of \(J_n\) contains a monomial
\[
\prod_{i\in A}x_i,
\qquad |A|=n-1,
\]
with nonzero coefficient. Under the diagonal specialization,
\[
(P\mathbf u)_i=p_{ii}u_i,
\]
and hence this monomial becomes
\[
\prod_{i\in A}p_{ii}u_i,
\]
which has degree \(2(n-1)\). Moreover, distinct monomials in the ANF
of \(J_n\) remain distinct after this substitution, so this monomial
cannot cancel. Therefore, the restricted function has degree at least
\(2(n-1)\), and consequently,
\[
\deg\bigl(J_n(P\mathbf u)\bigr)\ge 2(n-1).
\]
\end{proof}

Theorem~\ref{thm:degree-mixed} shows that including the basis transformation
at the input stage doubles the algebraic degree from $n-1$ to $2(n-1)$.
However, the mixed representation map $J_n(P\mathbf u)$ does not yet describe
the complete basis-dependent task because its output remains expressed in the
reference basis.

Recall from Lemma~\ref{InvBPB} that
\[
\operatorname{Inv}_B(\mathbf u)
=
P_B^{-1}J_n(P_B\mathbf u).
\]
Thus, when the basis matrix and the operand are treated jointly as inputs, the
complete raw formulation is
\[
(P,\mathbf u)
\longmapsto
P^{-1}J_n(P\mathbf u),
\qquad
P\in\operatorname{GL}_n(\mathbb F_2).
\]

To study its ANF, we extend this map from the invertible matrices to the full
Boolean matrix space. Such an extension is needed because the ANF is defined
on a full Boolean domain. Over $\mathbb F_2$, every invertible matrix has
determinant one, so
\[
P^{-1}=\operatorname{adj}(P)
\qquad
\text{for }P\in\operatorname{GL}_n(\mathbb F_2).
\]

\begin{definition}[Complete raw inversion map]
\label{def:complete-raw}
The complete raw inversion map is the polynomial map
\[
f_{\mathrm{raw}}
:
\mathbb F_2^{n\times n}\times V
\longrightarrow V
\]
defined by
\[
f_{\mathrm{raw}}(P,\mathbf u)
=
\operatorname{adj}(P)J_n(P\mathbf u).
\]
\end{definition}

For every valid change-of-basis matrix \(P_B\),
\[
f_{\mathrm{raw}}(P_B,\mathbf u)
=
\operatorname{Inv}_B(\mathbf u).
\]
The algebraic degree and ANF support studied below are defined for this
polynomial extension over the full Boolean matrix space. In the learning
experiments, we restrict the same map to invertible matrices, since only
these matrices represent valid changes of basis.

The output conversion introduces additional dependence on $P$ through
$\operatorname{adj}(P)$. Combining this with the degree of the mixed
representation map gives the following upper bound for the complete raw
formulation.

\begin{proposition}[Degree bound for the raw inversion task]
For \(n\ge2\), the polynomial representation of the raw inversion task
satisfies
\[
\deg(f_{\mathrm{raw}})\le 3(n-1).
\]
\end{proposition}
\begin{proof}
Each entry of \(\operatorname{adj}(P)\) is the determinant of an
\((n-1)\times(n-1)\) submatrix of \(P\), and therefore has algebraic
degree at most \(n-1\) in the entries of \(P\). By the preceding
theorem,
\[
\deg\bigl(J_n(P\mathbf u)\bigr)=2(n-1).
\]
Each coordinate of \(f_{\mathrm{raw}}\) is a sum of products of an
entry of \(\operatorname{adj}(P)\) and a coordinate of
\(J_n(P\mathbf u)\). Hence
\[
\deg(f_{\mathrm{raw}})
\le
(n-1)+2(n-1)
=
3(n-1).
\]
\end{proof}

Algebraic degree measures the largest interaction order appearing in a
Boolean map, but it does not describe how the variables are organized across
the joint ANF support. We next introduce a finer structural measure based on
the number of new variables that must be introduced at one step under an
optimal ordering of the joint ANF support.

\subsection{Joint ANF support and joint ANF leap}
We begin by collecting the ANF supports of all output coordinates into
a single joint support.

\begin{definition}[Joint ANF support]
Let
\[
G:\mathbb F_2^d\longrightarrow\mathbb F_2^m
\]
be a vector-valued Boolean map. Its coordinate ANFs can be written jointly as
\[
G(\mathbf x)
=
\bigoplus_{A\subseteq\{0,\ldots,d-1\}}
\mathbf c_A\prod_{i\in A}x_i,
\qquad
\mathbf c_A\in\mathbb F_2^m,
\]
where the XOR is taken componentwise in $\mathbb F_2^m$.

The \emph{joint ANF support} of $G$ is
\[
\mathcal A(G)
=
\left\{
A\subseteq\{0,\ldots,d-1\}:
\mathbf c_A\neq\mathbf 0
\right\}.
\]
Thus, $A\in\mathcal A(G)$ precisely when the monomial
\[
\prod_{i\in A}x_i
\]
appears with nonzero coefficient in at least one coordinate of $G$.
\end{definition}

Our next definition brings together two related but distinct perspectives on
interaction structure. Fourier analysis provides a standard way to describe
Boolean functions through their spectral support~\cite{ODonnell2014}. The
staircase and leap framework of Abbe et
al.~\cite{AbbeEtAl2022,AbbeEtAl2023} uses this Fourier support to describe
hierarchical structure in scalar-valued functions. Separately, Möbius-based
representations have been used to identify and recover higher-order
interactions among input variables~\cite{KangEtAl2024}. Our setting concerns
vector-valued polynomial maps over \(\mathbb F_2\), for which the ANF directly
records interactions among the basis and operand variables.

ANF support and Fourier support encode different objects, so the Fourier
definition of leap does not transfer directly to our setting. The ordering
idea can still be used. Given an ordering of the supported monomials, we
record how many variables in each monomial have not appeared earlier in the
ordering. For a vector-valued Boolean map, we apply this construction to the
joint ANF support. A monomial is included once if it appears in at least one
output coordinate.

To the best of our knowledge, this ANF based, vector-valued adaptation of leap has not been studied previously. We call the resulting measure the \emph{joint ANF leap}. %It describes how new variables enter the joint ANF support under an optimal ordering. The learning guarantees established in the Fourier setting do not automatically apply to this new measure. 
Its relationship with the observed learning
behavior is examined separately in Section~\ref{sec:experiments}.

\begin{definition}[Joint ANF leap]
Let
\[
G:\mathbb F_2^d\longrightarrow\mathbb F_2^m
\]
be a nonzero map, and let
\[
q=|\mathcal A(G)|.
\]
For an ordering
\[
A_1,\ldots,A_q
\]
of the sets in $\mathcal A(G)$, the number of variables introduced for the
first time at step $t$ is
\[
\left|
A_t\setminus\bigcup_{s<t}A_s
\right|.
\]
The \emph{joint ANF leap} of $G$ is
\[
L_{\mathrm{ANF}}^{\mathrm{joint}}(G)
=
\min_{(A_1,\ldots,A_q)}
\max_{1\leq t\leq q}
\left|
A_t\setminus\bigcup_{s<t}A_s
\right|,
\]
where the minimum is taken over all orderings of $\mathcal A(G)$.
\end{definition}

To determine which monomials belong to the joint ANF support, we use the
standard ANF coefficient formula. For
$A\subseteq\{0,\ldots,d-1\}$, the coefficient vector of the monomial
\[
\prod_{i\in A}x_i
\]
is
\[
\mathbf c_A
=
\bigoplus_{C\subseteq A}G(\mathbf 1_C),
\]
where $\mathbf 1_C\in\mathbb F_2^d$ is the indicator vector of $C$, and the
XOR is taken componentwise in $\mathbb F_2^m$. Hence,
\[
A\in\mathcal A(G)
\quad\Longleftrightarrow\quad
\bigoplus_{C\subseteq A}G(\mathbf 1_C)\neq\mathbf 0.
\]

We now compute the joint ANF leap for the same three formulations
considered in the degree analysis. We begin with the reference
inversion map, which provides the baseline for the
mixed representation and complete raw formulations.

\begin{theorem}[Joint ANF leap of reference inversion]
\label{thm:leap1}
For every \(n\ge2\),
\[
L_{\mathrm{ANF}}^{\mathrm{joint}}(J_n)=1.
\]
\end{theorem}

\begin{proof}
For each $j\in\{0,\ldots,n-1\}$, the coefficient vector corresponding to the
singleton monomial $x_j$ is
\[
\mathbf c_{\{j\}}
=
J_n(\mathbf 0)\oplus J_n(\mathbf e_j)
=
J_n(\mathbf e_j),
\]
where $\mathbf e_j$ is the $j$th standard basis vector of $V$.

Since $\mathbf e_j\neq\mathbf 0$, it represents a nonzero element of $F$.
Its inverse is also nonzero, and hence
\[
J_n(\mathbf e_j)\neq\mathbf 0.
\]
It follows that
\[
\{j\}\in\mathcal A(J_n)
\qquad
\text{for every }j.
\]

We may order the joint ANF support by placing the singleton sets first,
\[
\{0\},\{1\},\ldots,\{n-1\},
\]
followed by all remaining support sets in any order. Each singleton introduces
exactly one new variable. After these $n$ sets have been listed, every input
variable has already appeared, so each remaining support set introduces no
new variables. Thus,
\[
L_{\mathrm{ANF}}^{\mathrm{joint}}(J_n)\leq1.
\]

On the other hand, $J_n(\mathbf0)=\mathbf0$ and $J_n$ is nonzero. Hence, its
joint ANF support contains a nonempty set. The first support set in any
ordering introduces at least one variable, so
\[
L_{\mathrm{ANF}}^{\mathrm{joint}}(J_n)\geq1.
\]
Combining the two bounds gives
\[
L_{\mathrm{ANF}}^{\mathrm{joint}}(J_n)=1.
\]
\end{proof}

We next consider the mixed representation formulation
$J_n(P\mathbf u)$. Including the basis transformation at the input stage
changes not only the algebraic degree but also the way new variables enter
the joint ANF support.

\begin{theorem}[Joint ANF leap of the mixed representation formulation]
\label{thm:leap2}
For every \(n\ge2\),
\[
L_{\mathrm{ANF}}^{\mathrm{joint}}
\bigl(J_n(P\mathbf u)\bigr)
=
2.
\]
\end{theorem}

\begin{proof}
Let
\[
G:\mathbb F_2^{n\times n}\times V\longrightarrow V
\]
be defined by
\[
G(P,\mathbf u)=J_n(P\mathbf u).
\]
We consider the entries of $P$ and the coordinates of $\mathbf u$ jointly as
$n^2+n$ Boolean input variables.

We first prove the lower bound. Since
\[
G(P,\mathbf0)=J_n(\mathbf0)=\mathbf0
\]
for every $P$, the joint ANF support of $G$ contains no set consisting
only of variables from $P$. Similarly,
\[
G(0,\mathbf u)=J_n(\mathbf0)=\mathbf0
\]
for every $\mathbf u$, so the joint ANF support contains no set consisting
only of variables from $\mathbf u$.

Consequently, every set in $\mathcal A(G)$ contains at least one variable
$p_{ij}$ and at least one variable $u_k$. The first support set in any
ordering of $\mathcal A(G)$ introduces at least two variables. Hence,
\[
L_{\mathrm{ANF}}^{\mathrm{joint}}(G)\geq2.
\]

For the upper bound, let $E_{ij}$ denote the matrix whose $(i,j)$ entry is
one and whose remaining entries are zero, and let $\mathbf e_k$ denote the
$k$th standard basis vector of $V$.

By the ANF coefficient formula, the coefficient vector of the monomial
$p_{ij}u_k$ in $G$ is
\[
\begin{aligned}
\mathbf c_{\{p_{ij},u_k\}}
={}&
G(0,\mathbf0)
\oplus G(E_{ij},\mathbf0)
\oplus G(0,\mathbf e_k)
\oplus G(E_{ij},\mathbf e_k).
\end{aligned}
\]
The first three terms are zero. Moreover,
\[
E_{ij}\mathbf e_k
=
\begin{cases}
\mathbf e_i, & k=j,\\
\mathbf0, & k\neq j.
\end{cases}
\]
It follows that
\[
\mathbf c_{\{p_{ij},u_k\}}
=
\begin{cases}
J_n(\mathbf e_i), & k=j,\\
\mathbf0, & k\neq j.
\end{cases}
\]
Since $J_n(\mathbf e_i)\neq\mathbf0$, we obtain
\[
\{p_{ij},u_j\}\in\mathcal A(G)
\qquad
\text{for all }0\leq i,j\leq n-1.
\]

We order these support sets first, grouping them by their column index $j$:
\[
\begin{aligned}
&\{p_{00},u_0\},\{p_{10},u_0\},\ldots,
  \{p_{n-1,0},u_0\},\\
&\{p_{01},u_1\},\{p_{11},u_1\},\ldots,
  \{p_{n-1,1},u_1\},\\
&\hspace{3cm}\vdots\\
&\{p_{0,n-1},u_{n-1}\},\ldots,
  \{p_{n-1,n-1},u_{n-1}\}.
\end{aligned}
\]
For each $j$, the first set in the $j$th group introduces the two variables
$p_{0j}$ and $u_j$. Every subsequent set in the group introduces only one
new variable, namely $p_{ij}$.

After these sets have been listed, all entries of $P$ and all coordinates of
$\mathbf u$ have appeared. Every remaining set in $\mathcal A(G)$ therefore
introduces no new variables. This ordering introduces at most two new
variables at each step, so
\[
L_{\mathrm{ANF}}^{\mathrm{joint}}(G)\leq2.
\]

Combining the two bounds gives
\[
L_{\mathrm{ANF}}^{\mathrm{joint}}
\bigl(J_n(P\mathbf u)\bigr)=2.
\]
\end{proof}

We finally arrive at the  complete raw formulation $f_{\mathrm{raw}}$, which
also converts the output through $\operatorname{adj}(P)$. For this
formulation, the joint ANF leap grows at least linearly with the field
dimension. Unlike the preceding two results, which give exact values, the
following result establishes a lower bound.

\begin{theorem}[Lower bound on the joint ANF leap of the complete raw
formulation]
\label{thm:leap-raw}
For every $n\geq2$, each set
$A\in\mathcal A(f_{\mathrm{raw}})$ satisfies
\[
|A|\geq n.
\]
Equivalently, every monomial appearing in the joint ANF support of
\[
f_{\mathrm{raw}}(P,\mathbf u)
=
\operatorname{adj}(P)J_n(P\mathbf u)
\]
has degree at least $n$. Consequently,
\[
L_{\mathrm{ANF}}^{\mathrm{joint}}(f_{\mathrm{raw}})
\geq n.
\]
\end{theorem}

\begin{proof}
Let $A$ be a set of input
variables satisfying
\[
|A|<n.
\]
We show that its coefficient vector $\mathbf c_A$ is zero, and hence that
$A\notin\mathcal A(f_{\mathrm{raw}})$.

By the ANF coefficient formula,
\[
\mathbf c_A
=
\bigoplus_{C\subseteq A}
f_{\mathrm{raw}}(\mathbf1_C).
\]

Fix $C\subseteq A$. The indicator assignment $\mathbf1_C$ determines a
matrix $P_C$ and a vector $\mathbf u_C$ given by
\[
(P_C)_{ij}
=
\begin{cases}
1, & p_{ij}\in C,\\
0, & p_{ij}\notin C,
\end{cases}
\qquad
(\mathbf u_C)_k
=
\begin{cases}
1, & u_k\in C,\\
0, & u_k\notin C.
\end{cases}
\]
Thus, we identify $\mathbf1_C$ with the pair $(P_C,\mathbf u_C)$.

Let $k_P$ and $k_u$ denote the numbers of matrix variables and vector
variables contained in $C$, respectively. Then
\[
k_P+k_u=|C|<n.
\]

If $k_P\leq n-2$, the matrix $P_C$ has at most $n-2$ nonzero entries. Hence,
\[
\operatorname{rank}(P_C)\leq n-2,
\]
which implies
\[
\operatorname{adj}(P_C)=0.
\]
It follows that
\[
f_{\mathrm{raw}}(P_C,\mathbf u_C)
=
\operatorname{adj}(P_C)J_n(P_C\mathbf u_C)
=
\mathbf0.
\]

The only remaining possibility is $k_P=n-1$. Since
$k_P+k_u<n$, we must have $k_u=0$, and hence
\[
\mathbf u_C=\mathbf0.
\]
Consequently,
\[
J_n(P_C\mathbf u_C)
=
J_n(\mathbf0)
=
\mathbf0,
\]
and again
\[
f_{\mathrm{raw}}(P_C,\mathbf u_C)=\mathbf0.
\]

Every term in the coefficient formula is therefore zero, so
\[
\mathbf c_A=\mathbf0.
\]
Thus,
\[
A\notin\mathcal A(f_{\mathrm{raw}})
\qquad
\text{whenever }|A|<n.
\]

Finally, $f_{\mathrm{raw}}$ is nonzero because
\[
f_{\mathrm{raw}}(I,\mathbf u)=J_n(\mathbf u)
\]
and $J_n$ is nonzero. The first support set in any ordering of
$\mathcal A(f_{\mathrm{raw}})$ therefore contains at least $n$ variables.
Hence,
\[
L_{\mathrm{ANF}}^{\mathrm{joint}}(f_{\mathrm{raw}})
\geq n.
\]
\end{proof}

The results of this section give the following degree and joint ANF leap
values and bounds for the three formulations:
\[
\begin{aligned}
\text{algebraic degree:}\qquad
& n-1
\;\longrightarrow\;
2(n-1)
\;\longrightarrow\;
\leq 3(n-1),\\
\text{joint ANF leap:}\qquad
& 1
\;\longrightarrow\;
2
\;\longrightarrow\;
\geq n.
\end{aligned}
\]
From the reference formulation to the mixed representation formulation, both
the algebraic degree and the joint ANF leap increase. For the complete raw
formulation, the joint ANF leap is at least $n$, while the algebraic degree is
bounded above by $3(n-1)$.

\paragraph{Redundancy and complexity of the raw formulation.}

The contrast with the redundancy result is now precise. Let \(B\) and \(B'\)
be distinct bases in the same Galois orbit. Their basis matrices \(P_B\) and
\(P_{B'}\) are different, but they define the same function slice:
\[
f_{\mathrm{raw}}(P_B,\cdot)
=
\operatorname{Inv}_B
=
\operatorname{Inv}_{B'}
=
f_{\mathrm{raw}}(P_{B'},\cdot).
\]
Thus, each Galois orbit gives \(n\) distinct basis inputs whose associated
inversion maps are identical. This is the exact representation redundancy
established in Section~\ref{sec:math-frame}.

%Nevertheless, the complete raw formulation is not algebraically simple.
%When \(P\) and \(\mathbf u\) are treated jointly as variables, every nonzero
%monomial in its ANF has degree at least \(n\), and its joint ANF leap is at
%least \(n\).

The redundancy result and the complexity bounds describe different
properties. Redundancy compares the function slices obtained by fixing
different basis matrices, whereas structural complexity describes the full
joint map in the variables \((P,\mathbf u)\). Exact redundancy can therefore
coexist with a structurally complex raw formulation. We do not claim that
redundancy causes this complexity. Rather, the presence of redundancy alone
does not make the raw formulation algebraically simple.

This distinction motivates the experiments in the next section. We examine
whether redundancy within Galois orbits affects generalization to unseen
bases and whether the structural differences among the three formulations
are reflected in learning with finite data and model capacity.

\section{Experiments}
\label{sec:experiments}
We now examine two theoretical phenomena established in the preceding
sections. One is exact representation redundancy among bases in the same
Galois orbit. The other is the change in algebraic structure that occurs when
the basis transformation is included as part of the input. Our experimental
study combines exhaustive computation for finite cases with controlled
learning experiments that examine whether these theoretical differences are
reflected in learning performance under limited data and model capacity.

\subsection{Experimental setup}
\label{sec:experimental-setup}

\paragraph{Model.}
All learning experiments were implemented in Python using
PyTorch~\cite{PaszkeEtAl2019}. All learning experiments use a fully connected multilayer perceptron (MLP).
The model does not include Transformer layers or architectural components
designed to encode Galois symmetry. The MLP consists of three hidden layers
with ReLU activations and a final linear layer with $n$ output logits.
Experiments~1 and~2 use hidden width $256$, while Experiment~3 varies the
width as part of its scaling study.

We choose an MLP because each input is a short binary vector of fixed length
with no natural token or sequence structure that would call for a
Transformer. The MLP also receives no explicit information about the Galois
symmetry. This allows the comparison to focus on the effect of the
formulation without introducing an architectural bias designed around the
symmetry. Finally, a fully connected network is closer in spirit to the
models considered in the staircase and leap literature. Our architecture and
training procedure are not identical to those theoretical settings. Moreover, we do not claim that their learning guarantees apply directly here.

\paragraph{Input and output encoding.}
No tokenization or learned embedding is used. Each input consists of a basis
matrix
\[
P=(p_{ij})\in\operatorname{GL}_n(\mathbb F_2)
\]
and an operand
\[
\mathbf u=(u_0,\ldots,u_{n-1})^{\mathsf T}\in\mathbb F_2^n.
\]
We flatten $P$ in row major order and concatenate it with $\mathbf u$:
\[
\mathbf x(P,\mathbf u)
=
\bigl(
p_{00},p_{01},\ldots,p_{n-1,n-1},
u_0,\ldots,u_{n-1}
\bigr)^{\mathsf T}
\in\{0,1\}^{n^2+n}.
\]
The binary entries are supplied directly to the MLP as floating point values.
The input dimension is $12$ for $n=3$ and $20$ for $n=4$. The output consists
of $n$ logits, one for each coordinate of the target vector specified by the
corresponding formulation.

\paragraph{Training and evaluation.}
The model is trained using coordinatewise binary cross-entropy with logits.
We use AdamW~\cite{LoshchilovHutter2019} with learning rate $10^{-3}$ and parameters
\[
(\beta_1,\beta_2)=(0.9,0.999),\qquad
\epsilon=10^{-8},\qquad
\text{weight decay}=10^{-2},
\]
together with batches of size at most $512$. We use no learning rate
scheduler, warmup, early stopping, or gradient clipping. Model widths,
training budgets, splits into training and test sets, and numbers of random
seeds are given in the corresponding experiment subsections.

We report bit accuracy and exact accuracy. Bit accuracy is the proportion of
individual output coordinates predicted correctly. Exact accuracy counts an
example as correct only when all $n$ output bits are correct and is used as
the primary metric. Reported deviations are sample standard deviations across
random seeds.

\subsection{Experiment 0: Exact ANF computation}
\label{sec:exp0}

\paragraph{Purpose and design.}
This experiment verifies the ANF structure of the reference, mixed
representation, and complete raw formulations for $n=3$ and $n=4$. We
evaluate the reference map
\[
J_n(\mathbf u)
\]
over $\mathbb F_2^n$, and the other two maps
\[
J_n(P\mathbf u)
\qquad\text{and}\qquad
f_{\mathrm{raw}}(P,\mathbf u)
=
\operatorname{adj}(P)J_n(P\mathbf u)
\]
over
\[
\mathbb F_2^{n\times n}\times\mathbb F_2^n.
\]
The full Boolean matrix space is used because the ANF is defined on the
complete Boolean domain. This includes both invertible and singular values
of $P$ and agrees with the polynomial extension used in the definition of
$f_{\mathrm{raw}}$.

All input assignments are evaluated. The reference formulation requires $8$
evaluations for $n=3$ and $16$ for $n=4$. Each of the other two formulations
requires $2^{12}=4096$ evaluations for $n=3$ and
$2^{20}=1{,}048{,}576$ for $n=4$. The six computations contain
$2{,}105{,}368$ input evaluations in total.

For each output coordinate, we recover the ANF from its complete truth table
using the Boolean Möbius transform. We then compute the algebraic degree,
minimum positive monomial degree, joint ANF leap, and total coordinate-wise
monomial count
\[
N(G)
=
\sum_{k=0}^{m-1}
\left|
\operatorname{supp}_{\mathrm{ANF}}(G_k)
\right|.
\]
A monomial appearing in several output coordinates is counted once for each
coordinate in which it appears. No neural network training is used in this
experiment.

\paragraph{Results.}
The exact ANF statistics are reported in
Table~\ref{tab:exact-anf}.

\begin{table}[t]
\centering
\caption{Exact ANF statistics for the three formulations. Here \(N(G)\)
counts ANF monomials across output coordinates with multiplicity;
\(|\mathcal A(G)|\) is reported in
Table~\ref{tab:leap-certificates}.}
\label{tab:exact-anf}
\begin{tabular}{llrrrr}
\toprule
\(n\) & Formulation & Degree & Min. positive degree & Joint leap & \(N(G)\) \\
\midrule
3 & reference            & 2 & 1 & 1 & 9    \\
3 & mixed representation & 4 & 2 & 2 & 45   \\
3 & complete raw         & 6 & 3 & 3 & 168  \\
4 & reference            & 3 & 1 & 1 & 27   \\
4 & mixed representation & 6 & 2 & 2 & 552  \\
4 & complete raw         & 9 & 4 & 4 & 7008 \\
\bottomrule
\end{tabular}
\end{table}

\paragraph{Analysis.}
The exact values in Table~\ref{tab:exact-anf} agree with the theoretical
results of the preceding section. For both $n=3$ and $n=4$, the complete raw
formulation reaches the upper bound
\[
\deg(f_{\mathrm{raw}})=3(n-1)
\]
and the lower bound
\[
L_{\mathrm{ANF}}^{\mathrm{joint}}(f_{\mathrm{raw}})=n.
\]
Thus, both bounds are sharp for the two field sizes considered. The minimum
positive monomial degree and the total coordinate-wise monomial count $N(G)$
also increase across the three formulations.

These computations suggest that
\[
\deg(f_{\mathrm{raw}})=3(n-1),
\qquad
L_{\mathrm{ANF}}^{\mathrm{joint}}(f_{\mathrm{raw}})=n
\]
may hold for general $n$, although the cases $n=3$ and $n=4$ do not prove
either equality. Additional exact ANF degree profiles and coordinate
statistics are provided in Appendix~\ref{app:anf-details}.

\subsection{Experiment 1: Learning across the three formulations}
\label{sec:exp1}
\paragraph{Purpose and design.}
This experiment compares learning across the reference, mixed representation,
and complete raw formulations under matched conditions. A direct comparison
would create a domain mismatch. The reference map
\[
J_n:\mathbb F_2^n\longrightarrow\mathbb F_2^n
\]
has $2^n$ inputs, while the other two formulations take both a basis matrix
$P$ and an operand $\mathbf u$ as input. The input dimension, dataset size,
and split into training and test sets would therefore differ across the three
formulations.

To place all three maps on a common input domain, we lift the reference map by
defining
\[
\widetilde J_n(P,\mathbf u)=J_n(\mathbf u).
\]
The matrix $P$ is included as an input but does not affect the output. The
following proposition shows that this lifting preserves the ANF structure of
the reference map.

\begin{proposition}[ANF invariance under lifting]
\label{prop:lifted-reference}
Let $n\geq2$, and define
\[
\widetilde J_n:
\mathbb F_2^{n\times n}\times\mathbb F_2^n
\longrightarrow
\mathbb F_2^n
\]
by
\[
\widetilde J_n(P,\mathbf u)=J_n(\mathbf u).
\]
Then $\widetilde J_n$ and $J_n$ have the same algebraic degree, minimum
positive monomial degree, and degree profile. Under the natural
identification of the variables $u_0,\ldots,u_{n-1}$,
\[
\mathcal A(\widetilde J_n)=\mathcal A(J_n),
\qquad
N(\widetilde J_n)=N(J_n).
\]
They also have the same joint ANF leap. In particular,
\[
\deg(\widetilde J_n)=n-1,
\qquad
L_{\mathrm{ANF}}^{\mathrm{joint}}(\widetilde J_n)=1.
\]
\end{proposition}

\begin{proof}
Write the $k$th coordinate ANF of $J_n$ as
\[
J_{n,k}(\mathbf u)
=
\bigoplus_{S\subseteq\{0,\ldots,n-1\}}
c_{k,S}\prod_{i\in S}u_i.
\]
By definition,
\[
\widetilde J_{n,k}(P,\mathbf u)=J_{n,k}(\mathbf u).
\]
Hence, the ANF of $\widetilde J_{n,k}$ contains exactly the same monomials as
that of $J_{n,k}$, and no monomial contains a matrix variable $p_{ij}$. When
the support of $J_n$ is viewed in the enlarged variable set, all coordinate
ANF supports remain unchanged. The algebraic degree, minimum positive
monomial degree, and degree profile are therefore preserved, and
\[
\mathcal A(\widetilde J_n)=\mathcal A(J_n),
\qquad
N(\widetilde J_n)=N(J_n).
\]

Every ordering of the sets in $\mathcal A(J_n)$ gives an ordering of
$\mathcal A(\widetilde J_n)$ with the same number of new variables at each
step, and conversely. Thus,
\[
L_{\mathrm{ANF}}^{\mathrm{joint}}(\widetilde J_n)
=
L_{\mathrm{ANF}}^{\mathrm{joint}}(J_n).
\]
The stated values follow from
\[
\deg(J_n)=n-1,
\qquad
L_{\mathrm{ANF}}^{\mathrm{joint}}(J_n)=1.
\]
\end{proof}

We then compare the three maps
\begin{align}
f_{\mathrm{ref}}(P,\mathbf u)
    &=\widetilde J_n(P,\mathbf u)=J_n(\mathbf u),\\
f_{\mathrm{mix}}(P,\mathbf u)
    &=J_n(P\mathbf u),\\
f_{\mathrm{raw}}(P,\mathbf u)
    &=P^{-1}J_n(P\mathbf u)
\end{align}
on the common experimental domain
\[
\operatorname{GL}_n(\mathbb F_2)\times\mathbb F_2^n.
\]
Thus, all three formulations use the same input dimension, input pairs,
splits into training and test sets, model architecture, and optimization
budget. The ANF statement in
Proposition~\ref{prop:lifted-reference} is made on the full ambient Boolean
domain, whereas the learning experiment restricts $P$ to invertible matrices.

For each $n\in\{3,4\}$, we randomly divide the complete set of input pairs
$(P,\mathbf u)$ into $75\%$ training and $25\%$ test examples. For each random
seed, the same split indices are used for all three formulations. For $n=3$,
the domain contains $1344$ examples, giving $1008$ training examples and
$336$ test examples. For $n=4$, the domain contains $322{,}560$ examples,
giving $241{,}920$ training examples and $80{,}640$ test examples.

All models have three ReLU hidden layers of width $256$ and are trained for
$1000$ optimization steps. Each condition is repeated with five random
seeds, giving $3\times2\times5=30$ training runs in total. The remaining
model and optimization settings are those described in
Section~\ref{sec:experimental-setup}.

\paragraph{Results.}
Figure~\ref{fig:formulation-learning-curves} shows the held-out exact
accuracy throughout training for all five random seeds.
Table~\ref{tab:formulation-comparison} reports the final exact and bit
accuracies.

\begin{figure}[t]
\centering
\includegraphics[width=\textwidth]
{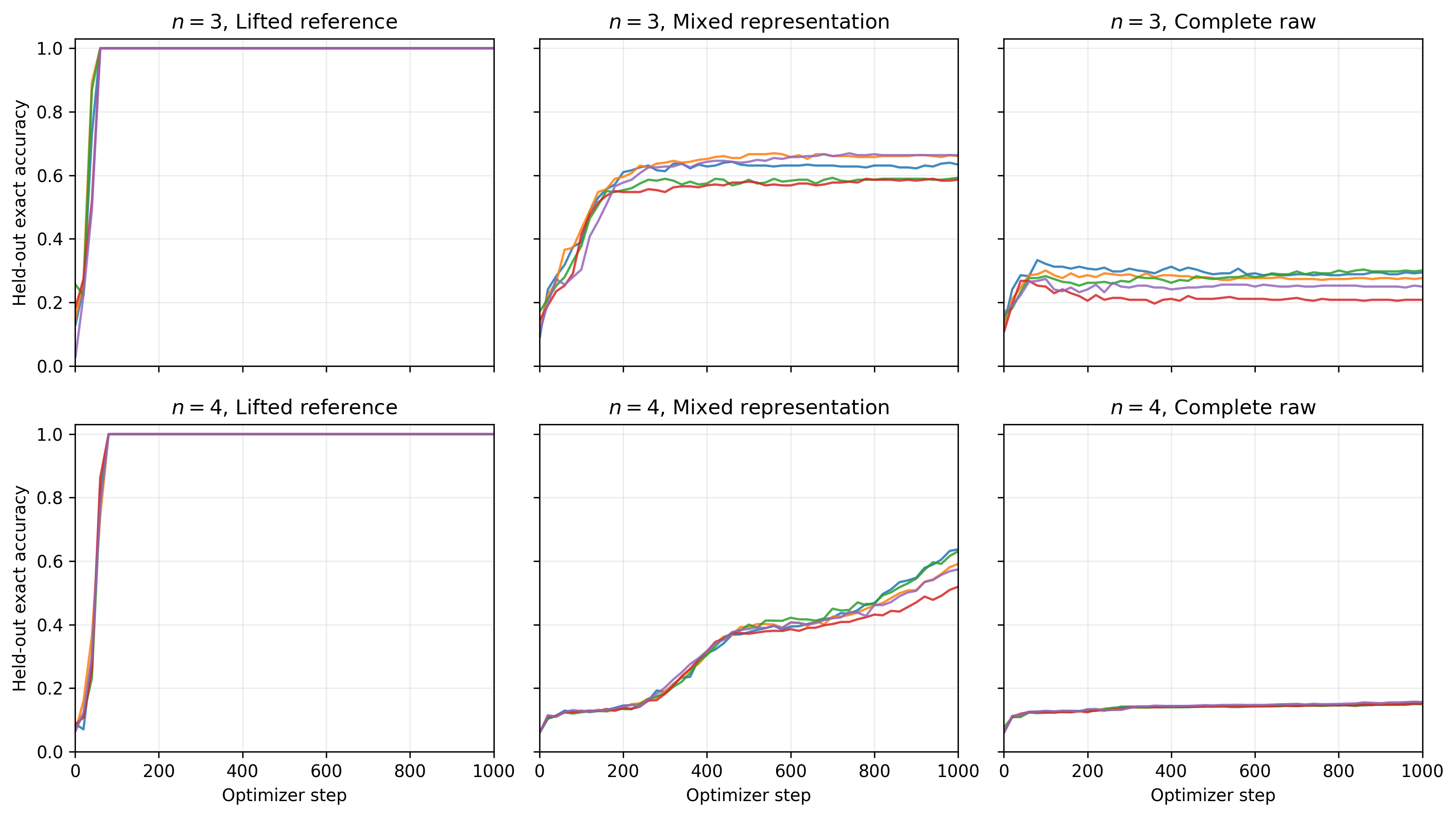}
\caption{Held-out exact accuracy for the lifted reference,
mixed representation, and complete raw formulations. Rows correspond to
\(n=3,4\), and each line represents one of five random seeds.}
\label{fig:formulation-learning-curves}
\end{figure}

\begin{table}[t]
\centering
\caption{Final test performance across the three formulations under matched
conditions. Results are means \(\pm\) sample standard deviations over five
random seeds.}
\label{tab:formulation-comparison}
\begin{tabular}{llcc}
\toprule
\(n\) & Formulation & Exact accuracy & Bit accuracy \\
\midrule
3 & reference
  & \(1.0000\pm0.0000\) & \(1.0000\pm0.0000\) \\
3 & mixed representation
  & \(0.6274\pm0.0367\) & \(0.8304\pm0.0205\) \\
3 & complete raw
  & \(0.2661\pm0.0378\) & \(0.6079\pm0.0226\) \\
4 & reference
  & \(1.0000\pm0.0000\) & \(1.0000\pm0.0000\) \\
4 & mixed representation
  & \(0.5905\pm0.0477\) & \(0.8385\pm0.0154\) \\
4 & complete raw
  & \(0.1528\pm0.0031\) & \(0.5928\pm0.0026\) \\
\bottomrule
\end{tabular}
\end{table}

\paragraph{Analysis.}
Figure~\ref{fig:formulation-learning-curves} and
Table~\ref{tab:formulation-comparison} show the same ordering for both field
sizes:
\[
\text{lifted reference}
\;>\;
\text{mixed representation}
\;>\;
\text{complete raw}.
\]
The same ordering holds for bit accuracy. Because the three formulations use
the same inputs, splits into training and test sets, model architecture, and
optimization budget, these differences cannot be attributed to unequal input
dimensions or test sets.

The lifted reference reaches perfect test accuracy for both $n=3$ and $n=4$.
Thus, including the irrelevant matrix $P$ does not prevent the MLP from
learning $J_n(\mathbf u)$. The mixed representation formulation achieves
lower accuracy than the lifted reference but remains much easier to learn
than the complete raw formulation. The gap between the mixed representation
and complete raw formulations is especially clear in exact accuracy for
$n=4$.

The held-out learning curves also reveal differences over the course of
optimization. The lifted reference reaches perfect held-out accuracy
within the first 100 optimization steps for both field sizes. For \(n=3\), the held-out accuracy of the mixed representation formulation
improves rapidly and then levels off, whereas that of the complete raw
formulation reaches a much lower plateau early in optimization. For $n=4$, the mixed
representation formulation continues to improve throughout the budget of
$1000$ steps, while the complete raw formulation improves only slowly.

These observations are consistent with the exact ANF results in
Table~\ref{tab:exact-anf}. By
Proposition~\ref{prop:lifted-reference}, the lifted reference has the same ANF
structure as the reference formulation. The reference formulation has the
smallest algebraic degree and joint ANF leap, the mixed representation
formulation has intermediate values, and the complete raw formulation has the
largest values for $n=3$ and $n=4$. As these ANF quantities increase, the
observed learning accuracy decreases under the tested conditions. %This is an
%empirical comparison and does not claim that algebraic degree or joint ANF
%leap alone determines neural network performance.

\subsection{Experiment 2: Generalization and representation redundancy}
\label{sec:exp2}

\paragraph{Purpose and design.}
Theorem~\ref{thm:exact-equivalence} shows that two bases induce the same
coordinate inversion map if and only if they belong to the same Galois orbit.
A random split of individual bases may therefore place different
representations of the same task in the training and test sets. This
experiment examines whether such overlap affects generalization for the
complete raw formulation.

We compare two splitting methods. In the \emph{random-basis split}, entire
basis matrices are assigned randomly to the training or test set, and all
$2^n$ operands associated with each basis remain together. A Galois orbit may contain bases from both sets. In the \emph{orbit-disjoint split},
entire Galois orbits are assigned to one set, so no Galois equivalent bases
appear in both sets.

Both methods use a $75/25$ division. For $n=3$, the $168$ bases form $56$
orbits of size $3$. The split contains $126$ training bases and $42$ test
bases, giving $1008$ training examples and $336$ test examples. In the
orbit-disjoint split, these bases correspond to $42$ training orbits and $14$
test orbits. For $n=4$, the $20{,}160$ bases form $5040$ orbits of size $4$.
The split contains $15{,}120$ training bases and $5040$ test bases, giving
$241{,}920$ training examples and $80{,}640$ test examples. The corresponding
orbit counts are $3780$ and $1260$.

For each field size and splitting method, we train an MLP with three hidden
layers of width $256$ for $1200$ optimization steps. Each condition is
repeated with ten random seeds, giving $40$ training runs in total. The
remaining settings are given in
Section~\ref{sec:experimental-setup}.

\paragraph{Results.}
Table~\ref{tab:orbit-generalization} reports the test performance under the
two splitting methods. 

\begin{table}[t]
\centering
\caption{Generalization under random-basis and Galois-orbit-disjoint splits
for the complete raw formulation. Results are means \(\pm\) sample standard
deviations over ten random seeds.}
\label{tab:orbit-generalization}
\begin{tabular}{llcc}
\toprule
\(n\) & Split & Exact accuracy & Bit accuracy \\
\midrule
3 & random basis
  & \(0.3089\pm0.0250\) & \(0.6388\pm0.0180\) \\
3 & orbit disjoint
  & \(0.2857\pm0.0096\) & \(0.6284\pm0.0211\) \\
4 & random basis
  & \(0.1591\pm0.0035\) & \(0.5980\pm0.0022\) \\
4 & orbit disjoint
  & \(0.1589\pm0.0040\) & \(0.5974\pm0.0032\) \\
\bottomrule
\end{tabular}
\end{table}

\paragraph{Analysis.}
For $n=3$, the random-basis split gives a mean exact accuracy $0.0232$ higher
than the orbit-disjoint split. Its mean bit accuracy is also $0.0104$ higher.
Allowing Galois equivalent bases to appear in both sets therefore gives a
modest improvement for the smaller field.

For $n=4$, the two splitting methods give nearly identical results. Their mean
exact accuracies differ by approximately $0.0002$, and their mean bit
accuracies differ by approximately $0.0006$. Both differences are smaller
than the variation across random seeds. We find no clear benefit from overlap
between equivalent tasks at $n=4$ under the tested model and training budget.

The theorem gives an exact redundancy factor of $n$ among basis
representations, but this redundancy does not necessarily produce a large
improvement in test performance. Its effect is modest for $n=3$ and is not
detectable for $n=4$ in this experiment.

\subsection{Experiment 3: Scaling with data and model width}
\label{sec:exp3}

\paragraph{Purpose and design.}
This experiment examines how the three formulations respond to increases in
training data and model width. We focus on \(n=4\), the larger and more difficult of the two field sizes
considered in Experiment~1, for which the complete input domain
\(\mathrm{GL}_4(\mathbb F_2)\times\mathbb F_2^4\) contains
\[
|\mathrm{GL}_4(\mathbb F_2)|\cdot2^4
=
20{,}160\cdot16
=
322{,}560
\]
input pairs. For each random seed, we divide this domain into
\(241{,}920\) training examples and \(80{,}640\) test examples. The test set
is held fixed, while the training set is subsampled using fractions
\[
25\%,\qquad 50\%,\qquad 100\%.
\]
These fractions correspond to
\[
60{,}480,\qquad 120{,}960,\qquad 241{,}920
\]
training examples, respectively. Within each seed, the same split and nested
training subsets are used for all three formulations and all model widths.

We compare the lifted reference, mixed representation, and complete raw
formulations. The lifted reference is defined in
Proposition~\ref{prop:lifted-reference}. We vary the hidden width over
\[
128,\qquad 256,\qquad 512.
\]
Each model is trained for five epochs. With batch size \(512\), the three training fractions correspond to \(595\), \(1185\), and \(2365\)
optimization steps, respectively. Thus, the number of epochs is fixed, while
larger training sets provide both more examples and more optimization steps.

Each combination of formulation, data fraction, and model width is repeated
with four random seeds, giving
\[
3\text{ formulations}\times
3\text{ data fractions}\times
3\text{ widths}\times
4\text{ seeds}
=
108
\]
training runs. The remaining model and optimization settings are those
described in Section~\ref{sec:experimental-setup}.

\paragraph{Results.}
Figure~\ref{fig:scaling} reports the held-out exact accuracy for all
\(27\) combinations of formulation, data fraction, and model width.

\begin{figure}[t]
\centering
\includegraphics[width=\textwidth]
{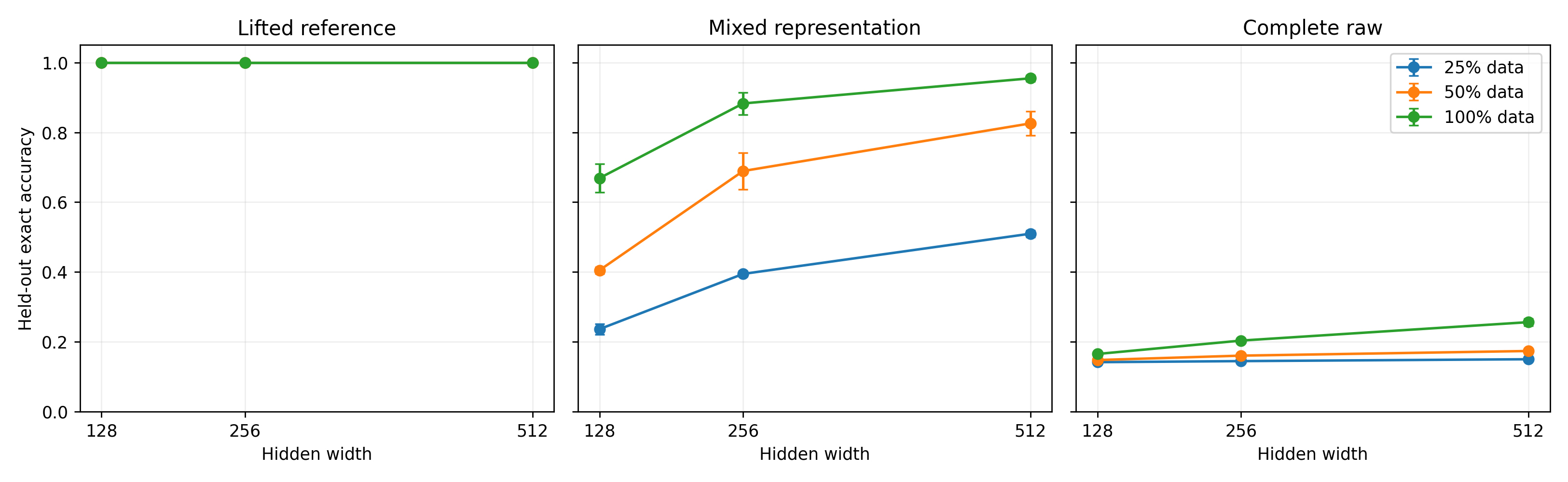}
\caption{Mean held-out exact accuracy at $n=4$ across model widths and training
fractions. Error bars show sample standard deviations over four seeds.}
\label{fig:scaling}
\end{figure}

The lifted reference reaches perfect exact accuracy in every condition. For the
mixed representation formulation with the full training set, increasing the
width from \(128\) to \(512\) raises exact accuracy from
\[
0.6688\pm0.0413
\quad\text{to}\quad
0.9546\pm0.0084.
\]
For the complete raw formulation, the corresponding increase is from
\[
0.1655\pm0.0027
\quad\text{to}\quad
0.2566\pm0.0100.
\]

\paragraph{Analysis.}

The lifted reference remains easy to learn across all training fractions
and model widths. Its perfect accuracy shows that the MLP can consistently
ignore the irrelevant matrix input and learn \(J_4(\mathbf u)\).

The mixed representation formulation improves with greater model width.
Its performance also improves when a larger training set is used while
the number of training epochs is fixed at five. At the largest setting,
it reaches approximately \(95.5\%\) exact accuracy. The complete raw
formulation also improves, but much more slowly. Even with the full
training set and width \(512\), its exact accuracy is approximately
\(25.7\%\).

These results extend the comparison in Experiment~1. Increasing the model
width and the training fraction under the five epoch protocol largely
closes the gap between the mixed representation formulation and the lifted
reference. In contrast, the complete raw formulation remains difficult
over the tested range. This behavior is consistent with its larger computed
algebraic degree and joint ANF leap for \(n=4\).

Because the number of optimization steps increases with the training
fraction, the effect of additional data cannot be separated completely
from the effect of additional optimization. The results should therefore
be interpreted as scaling under a training protocol of five epochs, rather
than as a controlled comparison with a fixed number of optimization steps.

\section{Discussion}
\label{sec:discussion}

\paragraph{Exact redundancy and empirical generalization.}
Theorem~\ref{thm:exact-equivalence} and Experiment~2 describe two
aspects of representation redundancy. The theorem shows that each coordinate
inversion task has exactly \(n\) basis representations, with the equivalence
classes given by Galois orbits. This is an exact property of the task family,
but it does not imply that a learner can identify or use the equivalence from
raw basis matrices.

Experiment~2 makes this distinction visible. Allowing bases from the same Galois orbit to appear in both the training and test sets gives only
a modest improvement for \(n=3\) and no clear improvement for \(n=4\) under
the tested conditions. Hence, exact representation redundancy can be present
without producing a large generalization gain for a standard MLP. Using this
redundancy more effectively may require a representation or model architecture
that makes the Galois action explicit.

\paragraph{ANF structure and learning difficulty.}
The theoretical and experimental results agree on the ordering of the
three formulations. The reference formulation has the smallest algebraic
degree and joint ANF leap and is learned easily. The mixed representation
formulation has intermediate values and improves strongly with additional
data and model width. For \(n=3\) and \(n=4\), the exact ANF computations assign the largest
degree and joint ANF leap to the complete raw formulation. In the learning
experiments, this formulation also achieves the lowest accuracy under all
tested conditions.

The comparison between the mixed representation and complete raw formulations
helps locate where this difficulty enters. The map
\[
J_n(P\mathbf u)
\]
uses \(P\) only to transform the input, while
\[
P^{-1}J_n(P\mathbf u)
\]
also converts the output back to the original basis. The theoretical results give bounds on the algebraic degree and joint ANF
leap. For \(n=3\) and \(n=4\), the exact computations reach both bounds.
The learning experiments also show a clear separation between the two
formulations, especially as the training fraction and model width increase.

This agreement supports the use of algebraic degree and joint ANF leap as
descriptions of the structure introduced by representation exposure. It does
not imply that either quantity alone determines neural-network performance.
In particular, our joint ANF leap is defined using the joint ANF support of a
vector-valued Boolean map, while the original staircase and leap results
concern Fourier support and specific learning settings. The experiments
provide evidence of a connection, but they are not a direct application of
those learning guarantees. 

The ANF analysis is performed on the polynomial extension defined using
the adjugate over the full Boolean matrix space, whereas the learning
experiments are restricted to invertible basis matrices. Establishing an
intrinsic notion of ANF complexity for the restricted domain is left for
future work.

\section{Conclusion}
\label{sec:conclusion}

We studied how basis choice affects finite-field inversion from two points of
view: exact representation redundancy and the structure of the joint map
obtained when the basis is included as part of the input. We proved that two
ordered bases induce the same coordinate inversion map if and only if they
belong to the same Galois orbit. Since the Galois action is free, every
inversion task has exactly $n$ basis representations.

We also compared the reference, mixed representation, and complete raw
formulations through algebraic degree and joint ANF leap. The reference
formulation has degree $n-1$ and joint ANF leap $1$, while the mixed
representation formulation has degree $2(n-1)$ and joint ANF leap $2$. For
the complete raw formulation, %every monomial in the joint ANF support has
%degree at least $n$, 
the algebraic degree is at most $3(n-1)$, and the joint
ANF leap is at least $n$. These
results show that exposing the basis transformation alters the ANF
structure of the task in ways that are not captured by the exact
redundancy count.

The experiments are consistent with the theoretical analysis. Exact
computations for $n=3$ and $n=4$ attain the raw degree and joint ANF leap
bounds. In the learning experiments, accuracy decreases from the lifted
reference to the mixed representation and complete raw formulations.
Increasing the amount of training data and the model width greatly improves
the mixed representation formulation, while the complete raw formulation
remains more difficult over the tested range. The exact redundancy within
Galois orbits gives only a modest generalization benefit for $n=3$ and no
clear benefit for $n=4$ under the tested conditions.

The main conclusion is that exact representation redundancy can coexist with
a joint map that is difficult to learn from raw inputs. Several inputs may
represent the same task without making that equivalence easy for a standard
learner to identify or use. The present exact computations are limited to
$n=3$ and $n=4$, and the learning experiments use a standard MLP without
explicit Galois symmetry. The scaling experiment also fixes the number of
epochs, so the effects of additional data and additional optimization steps
are not separated. Proving the exact raw ANF values for general $n$,
separating the effect of joint ANF leap from algebraic degree, and studying
models that explicitly use the Galois action are natural directions for
future work.

%%%%%%%%%%%%%%%%%%%%%%%%%%%%%%%%%%%%%%%%%%

%%%%%%%%%%%%%%%%%%%%%%%%%%%%%%%%%%%%%%%%%%
\vspace{6pt}

\appendix
\section{Additional exact ANF results}
\label{app:anf-details}
This appendix gives the degree profiles, coordinate-level monomial counts,
joint ANF support sizes, and joint ANF leap values from
Experiment~0. Let
\[
G=(G_0,\ldots,G_{m-1}):
\mathbb F_2^d\longrightarrow\mathbb F_2^m
\]
be a vector-valued Boolean map with coordinate ANFs
\[
G_k(\mathbf x)
=
\bigoplus_{A\subseteq\{0,\ldots,d-1\}}
c_{k,A}\prod_{i\in A}x_i.
\]
For $0\leq\ell\leq d$, let $N_\ell(G)$ denote the total number of ANF
monomials of degree $\ell$ across all output coordinates:
\[
N_\ell(G)
=
\sum_{k=0}^{m-1}
\left|
\left\{
A\subseteq\{0,\ldots,d-1\}:
c_{k,A}\neq0,\ |A|=\ell
\right\}
\right|.
\]
A monomial appearing in multiple output coordinates is counted once for each
coordinate in which it appears. The total coordinate-wise monomial count is
\[
N(G)
=
\sum_{\ell=0}^{d}N_\ell(G)
=
\sum_{k=0}^{m-1}
\left|
\operatorname{supp}_{\mathrm{ANF}}(G_k)
\right|.
\]

For each output coordinate, the ANF coefficients are recovered from the
complete truth table on the ambient Boolean domain using the Boolean Möbius
transform
\[
c_{k,A}
=
\bigoplus_{C\subseteq A}G_k(\mathbf1_C),
\]
where $\mathbf1_C\in\mathbb F_2^d$ is the indicator vector of $C$. All values
reported below are exact.

\subsection{Degree profiles}
Table~\ref{tab:complete-degree-profiles} reports the complete degree profiles.
An entry \(\ell:c\) means that \(N_\ell(G)=c\), with monomials appearing in
multiple output coordinates counted once for each coordinate.

\begin{table}[t]
\centering
\caption{Complete ANF degree profiles. Counts are summed over the output
coordinates.}
\label{tab:complete-degree-profiles}
\begin{tabular}{lll}
\toprule
\(n\) & Formulation & Degree profile \(\ell:c\) \\
\midrule
3 & reference
  & \(1:6,\;2:3\) \\
3 & mixed representation
  & \(2:18,\;3:9,\;4:18\) \\
3 & complete raw
  & \(3:24,\;4:66,\;5:42,\;6:36\) \\
\midrule
4 & reference
  & \(1:10,\;2:12,\;3:5\) \\
4 & mixed representation
  & \(2:40,\;3:48,\;4:164,\;5:180,\;6:120\) \\
4 & complete raw
  & \(4:168,\;5:792,\;6:1440,\;7:2280,\;8:1608,\;9:720\) \\
\bottomrule
\end{tabular}
\end{table}

The profiles show that the difference among the formulations is not limited
to their largest degrees. The mixed representation and complete raw
formulations contain monomials over wider degree ranges, and the complete raw
formulation has no monomial below degree \(n\) for either field size.

\subsection{Coordinate-level ANF statistics}
The preceding degree profiles combine the monomial counts from all output
coordinates. We now report the corresponding coordinate-level results. Output
coordinate \(k\) is the coefficient of \(t^k\) in the fixed polynomial basis:
\[
a_0+a_1t+\cdots+a_{n-1}t^{n-1}
\longleftrightarrow
(a_0,a_1,\ldots,a_{n-1})^{\mathsf T}.
\]
For each coordinate, we report both its degree profile and its total number of
ANF monomials. As above, an entry \(\ell:c\) means that the coordinate contains
\(c\) monomials of degree \(\ell\). Tables~\ref{tab:coordinate-anf-n3} and
\ref{tab:coordinate-anf-n4} report these statistics for \(n=3\) and \(n=4\),
respectively.

\begin{table}[t]
\centering
\caption{Coordinate-level ANF statistics for \(n=3\).}
\label{tab:coordinate-anf-n3}
\small
\begin{tabular}{lllr}
\toprule
Formulation & Output & Degree profile \(\ell:c\) & Monomials \\
\midrule
reference
  & \(0\) & \(1:3,\;2:1\) & 4 \\
reference
  & \(1\) & \(1:1,\;2:1\) & 2 \\
reference
  & \(2\) & \(1:2,\;2:1\) & 3 \\
\midrule
mixed representation
  & \(0\) & \(2:9,\;3:3,\;4:6\) & 18 \\
mixed representation
  & \(1\) & \(2:3,\;3:3,\;4:6\) & 12 \\
mixed representation
  & \(2\) & \(2:6,\;3:3,\;4:6\) & 15 \\
\midrule
complete raw
  & \(0\) & \(3:8,\;4:22,\;5:14,\;6:12\) & 56 \\
complete raw
  & \(1\) & \(3:8,\;4:22,\;5:14,\;6:12\) & 56 \\
complete raw
  & \(2\) & \(3:8,\;4:22,\;5:14,\;6:12\) & 56 \\
\bottomrule
\end{tabular}
\end{table}

For \(n=3\), Table~\ref{tab:coordinate-anf-n3} shows that the reference and
mixed representation formulations have different monomial counts across
their output coordinates. In contrast, the three output coordinates of the
complete raw formulation have identical degree profiles, each containing
\(56\) monomials. Summing the coordinate counts gives \(9\), \(45\), and
\(168\) monomials for the three formulations, respectively.

\begin{table}[t]
\centering
\caption{Coordinate-level ANF statistics for \(n=4\).}
\label{tab:coordinate-anf-n4}
\small
\begin{tabular}{lllr}
\toprule
Formulation & Output & Degree profile \(\ell:c\) & Monomials \\
\midrule
reference
  & \(0\) & \(1:4,\;2:2,\;3:2\) & 8 \\
reference
  & \(1\) & \(1:1,\;2:4,\;3:1\) & 6 \\
reference
  & \(2\) & \(1:2,\;2:3,\;3:1\) & 6 \\
reference
  & \(3\) & \(1:3,\;2:3,\;3:1\) & 7 \\
\midrule
mixed representation
  & \(0\) & \(2:16,\;3:8,\;4:32,\;5:72,\;6:48\) & 176 \\
mixed representation
  & \(1\) & \(2:4,\;3:16,\;4:52,\;5:36,\;6:24\) & 132 \\
mixed representation
  & \(2\) & \(2:8,\;3:12,\;4:40,\;5:36,\;6:24\) & 120 \\
mixed representation
  & \(3\) & \(2:12,\;3:12,\;4:40,\;5:36,\;6:24\) & 124 \\
\midrule
complete raw
  & \(0\) & \(4:42,\;5:198,\;6:360,\;7:570,\;8:402,\;9:180\)
  & 1752 \\
complete raw
  & \(1\) & \(4:42,\;5:198,\;6:360,\;7:570,\;8:402,\;9:180\)
  & 1752 \\
complete raw
  & \(2\) & \(4:42,\;5:198,\;6:360,\;7:570,\;8:402,\;9:180\)
  & 1752 \\
complete raw
  & \(3\) & \(4:42,\;5:198,\;6:360,\;7:570,\;8:402,\;9:180\)
  & 1752 \\
\bottomrule
\end{tabular}
\end{table}

For \(n=4\), Table~\ref{tab:coordinate-anf-n4} shows that the same pattern is
stronger. The reference formulation contains between \(6\) and \(8\)
monomials per coordinate, while the mixed representation formulation contains
between \(120\) and \(176\). Every coordinate of the complete raw formulation
contains \(1752\) monomials with the same degree profile. Summing the
coordinate counts gives the totals \(27\), \(552\), and \(7008\) reported in
Table~\ref{tab:exact-anf}.

The equality of the coordinate-level profiles in the complete raw formulation
is an observed property of the \(n=3\) and \(n=4\) computations. We do not
claim that this equality holds for arbitrary \(n\) or for every choice of
reference basis.

\subsection{Joint ANF leap verification}
The joint ANF leap is computed from the deduplicated joint ANF support
\(\mathcal A(G)\), while \(N(G)\) counts monomials separately across output
coordinates. For each formulation, the minimum positive monomial degree gives
a lower bound on the joint ANF leap, since the first nonempty support set in
any ordering must introduce at least that many variables.

For the reference and mixed representation formulations, the exact joint ANF
leaps follow from Theorems~\ref{thm:leap1} and
\ref{thm:leap2}. For the complete raw formulation, the theoretical result
gives the lower bound
\[
L_{\mathrm{ANF}}^{\mathrm{joint}}(f_{\mathrm{raw}})\geq n.
\]
To verify that this bound is attained for \(n=3\) and \(n=4\), we construct
an initial sequence of sets in \(\mathcal A(f_{\mathrm{raw}})\) such that each
set introduces at most \(n\) new variables and the sequence covers every
input variable. Once all input variables have appeared, the remaining sets
in the joint ANF support may be appended in any order without introducing
new variables. Table~\ref{tab:raw-leap-certificates} gives these explicit
certificate sequences.

\begin{table}[t]
\centering
\caption{Explicit joint ANF leap certificates for the complete raw
formulation. Each listed support set belongs to
\(\mathcal A(f_{\mathrm{raw}})\).}
\label{tab:raw-leap-certificates}
\small
\begin{tabular}{ccp{0.34\textwidth}p{0.30\textwidth}c}
\toprule
\(n\) & Step & Support set
& Newly introduced variables & Count \\
\midrule
3 & 1
& \(\{p_{01},p_{10},u_0\}\)
& \(\{p_{01},p_{10},u_0\}\) & 3 \\

3 & 2
& \(\{p_{00},p_{11},u_1\}\)
& \(\{p_{00},p_{11},u_1\}\) & 3 \\

3 & 3
& \(\{p_{12},p_{20},u_2\}\)
& \(\{p_{12},p_{20},u_2\}\) & 3 \\

3 & 4
& \(\{p_{02},p_{21},p_{22},u_2\}\)
& \(\{p_{02},p_{21},p_{22}\}\) & 3 \\
\midrule
4 & 1
& \(\{p_{02},p_{10},p_{33},u_3\}\)
& \(\{p_{02},p_{10},p_{33},u_3\}\) & 4 \\

4 & 2
& \(\{p_{03},p_{11},p_{32},u_2\}\)
& \(\{p_{03},p_{11},p_{32},u_2\}\) & 4 \\

4 & 3
& \(\{p_{12},p_{23},p_{30},u_0\}\)
& \(\{p_{12},p_{23},p_{30},u_0\}\) & 4 \\

4 & 4
& \(\{p_{00},p_{13},p_{21},u_1\}\)
& \(\{p_{00},p_{13},p_{21},u_1\}\) & 4 \\

4 & 5
& \(\{p_{00},p_{01},p_{20},p_{22},p_{31},u_0,u_1\}\)
& \(\{p_{01},p_{20},p_{22},p_{31}\}\) & 4 \\
\bottomrule
\end{tabular}
\end{table}

For \(n=3\), the four sets in
Table~\ref{tab:raw-leap-certificates} cover all \(12\) input variables and
introduce at most \(3\) new variables at each step. Hence,
\[
L_{\mathrm{ANF}}^{\mathrm{joint}}(f_{\mathrm{raw}})\leq3.
\]
Along with the theoretical lower bound, this gives
\[
L_{\mathrm{ANF}}^{\mathrm{joint}}(f_{\mathrm{raw}})=3.
\]

For \(n=4\), the five listed sets cover all \(20\) input variables. The final
set contains seven variables, but \(p_{00}\), \(u_0\), and \(u_1\) have
already appeared, so it introduces only four new variables. Therefore,
\[
L_{\mathrm{ANF}}^{\mathrm{joint}}(f_{\mathrm{raw}})\leq4.
\]
Combining this certificate with the lower bound gives
\[
L_{\mathrm{ANF}}^{\mathrm{joint}}(f_{\mathrm{raw}})=4.
\]

Thus, the exact joint ANF leaps are \(1\), \(2\), and \(3\) for \(n=3\),
and \(1\), \(2\), and \(4\) for \(n=4\), as summarized in
Tables~\ref{tab:leap-certificates} and \ref{tab:exact-anf}.

\begin{table}[t]
\centering
\caption{Exact verification of the joint ANF leap. The column
\(|\mathcal A(G)|\) gives the cardinality of the deduplicated joint ANF
support.}
\label{tab:leap-certificates}
\begin{tabular}{llrrr}
\toprule
\(n\) & Formulation & \(|\mathcal A(G)|\) & Joint ANF leap & Certificate length \\
\midrule
3 & reference            & 6    & 1 & 3  \\
3 & mixed representation & 36   & 2 & 6  \\
3 & complete raw         & 168  & 3 & 4  \\
4 & reference            & 14   & 1 & 4  \\
4 & mixed representation & 368  & 2 & 10 \\
4 & complete raw         & 6504 & 4 & 5  \\
\bottomrule
\end{tabular}
\end{table}

\clearpage

% Please provide either the correct journal abbreviation (e.g. according to the “The list of Title Word Abbreviations” https://portal.issn.org/ltwa) or the full name of the journal. 
% Citations and references in the Supplementary Materials are permitted provided that they also appear in the reference list here. 


\bibliographystyle{apalike}
\bibliography{references}
\clearpage

%=====================================
% References, variant B: internal bibliography
%=====================================

% If authors have biography, please use the format below
%\section*{Short Biography of Authors}
%\bio
%{\raisebox{-0.35cm}{\includegraphics[width=3.5cm,height=5.3cm,clip,keepaspectratio]{Definitions/author1.pdf}}}
%{\textbf{Firstname Lastname} Biography of first author}
%
%\bio
%{\raisebox{-0.35cm}{\includegraphics[width=3.5cm,height=5.3cm,clip,keepaspectratio]{Definitions/author2.jpg}}}
%{\textbf{Firstname Lastname} Biography of second author}

% For the MDPI journals use author-date citation, please follow the formatting guidelines on http://www.mdpi.com/authors/references
% To cite two works by the same author: \citeauthor{ref-journal-1a} (\citeyear{ref-journal-1a}, \citeyear{ref-journal-1b}). This produces: Whittaker (1967, 1975)
% To cite two works by the same author with specific pages: \citeauthor{ref-journal-3a} (\citeyear{ref-journal-3a}, p. 328; \citeyear{ref-journal-3b}, p.475). This produces: Wong (1999, p. 328; 2000, p. 475)

%%%%%%%%%%%%%%%%%%%%%%%%%%%%%%%%%%%%%%%%%%
%% for journal Sci
%\reviewreports{\\
%Reviewer 1 comments and authors’ response\\
%Reviewer 2 comments and authors’ response\\
%Reviewer 3 comments and authors’ response
%}
%%%%%%%%%%%%%%%%%%%%%%%%%%%%%%%%%%%%%%%%%%

%\isPreprints{}{% This command is only used for ``preprints''.
\end{document}